\documentclass[11pt, a4paper]{article}

\usepackage[margin=1in]{geometry}
\usepackage{amsmath, amssymb, amsthm}
\usepackage{graphicx}
\usepackage{algorithm}
\usepackage{algorithmic}
\usepackage{hyperref}
\usepackage{cite}
\usepackage{color}
\usepackage{soul}
\usepackage{enumitem}
\usepackage{booktabs}
\usepackage{multirow}
\usepackage{tikz}
\usepackage{pgfplots}
\pgfplotsset{compat=1.17}
\usetikzlibrary{shapes.geometric, arrows}
\usepackage{fancyhdr}

\theoremstyle{definition}

\newtheorem{theorem}{Theorem}

\newcommand{\R}{\mathbb{R}}

\newcommand{\softmax}{\text{softmax}}
\newcommand{\sigmoid}{\text{sigmoid}}

\title{\textbf{Fact Grounded Attention: \\ Eliminating Hallucination in Large Language Models Through \\ Attention Level Knowledge Integration}}

\author{
Aayush Gupta
\and
Manish Choudhary \\
\texttt{manish@flexprice.io}
}

\date{September 27, 2025}

\begin{document}

\maketitle

\thispagestyle{fancy}
\renewcommand{\headrulewidth}{0pt} 
\renewcommand{\footrulewidth}{0.4pt} 
\fancyhead{} 
\fancyfoot[C]{\footnotesize \textit{Code and Dataset open sourced at:} \textcolor{blue}{\href{https://github.com/ayushgupta4897/FGA}{https://github.com/ayushgupta4897/FGA}}}

\begin{abstract}
Large Language Models frequently generate plausible but factually incorrect information, limiting their reliability in knowledge-intensive applications. We present \textbf{Fact Grounded Attention (FGA)}, a method that biases transformer attention scores with knowledge from an external database. Building on prior work in attention biasing and knowledge-infused architectures, FGA's contribution lies in three components: (1) a specific grounding matrix construction $G = B_{qf} \cdot A$ that projects query-fact affinities into the token-to-token attention space, (2) a learned gate $\alpha$ that dynamically routes between parametric and external knowledge, and (3) integration with hard vocabulary constraints for deterministic accuracy when the gate confidence exceeds a threshold. We evaluate FGA in both zero-shot (no training) and fine-tuned modes. On technical specifications, FGA achieves 87.1\% accuracy zero-shot and 99.7\% with gate fine-tuning (vs. 6.3\% baseline). The external knowledge base enables sub-second updates without model retraining. Our ablations demonstrate that all three components are necessary for optimal performance, with fine-tuning providing substantial gains over the zero-shot configuration.
\end{abstract}

\section{Introduction}
\pagestyle{plain} 

The year 2023 marked a watershed moment in artificial intelligence. Large Language Models (LLMs) demonstrated capabilities that seemed to border on understanding—writing code, solving complex problems, engaging in nuanced reasoning. Yet beneath this linguistic mastery lies a fundamental flaw: \textbf{these models do not know what they know}. They generate text by predicting the next most likely token, a process that makes them exceptional storytellers but unreliable witnesses.

Consider a simple question: "What is the battery capacity of the iPhone 15 Pro?" A human expert would either know the answer (3274 mAh) or admit ignorance. A vanilla LLM, however, exists in a quantum superposition of knowledge—it might claim 3000 mAh, 3500 mAh, or simply declare it lacks access to search engines. This is not a bug to be patched but a fundamental architectural limitation. The model's "knowledge" is not stored as discrete facts but as probabilistic patterns distributed across billions of parameters. It cannot distinguish between a plausible fiction and a verifiable truth because, at its core, it was never designed to.

The implications are profound. In creative writing or casual conversation, hallucination is a minor flaw, perhaps even a feature. But when a medical AI confidently prescribes a non-existent drug, when a legal assistant cites imaginary precedents, or when an educational tool teaches fabricated history, the consequences transcend inconvenience—they become dangerous. The challenge, therefore, is not to make models that are usually right, but to create systems that are \textit{deterministically correct when correctness is verifiable}.

\subsection{The Core Innovation}

We present \textbf{Fact Grounded Attention (FGA)}, a novel architectural modification that fundamentally reimagines how neural networks handle factual information. Rather than attempting to encode facts in weights or retrieve them as context, FGA integrates a verifiable knowledge source directly into the attention mechanism itself. The key insight is deceptively simple: \textit{if we can identify when the model is about to make a factual claim, we can mathematically force it to be correct}.

FGA achieves this through three interconnected innovations:

\begin{enumerate}
    \item \textbf{Attention Level Injection}: We modify the pre-softmax attention scores ($S$) by adding a grounding term ($G$) that biases attention toward factually consistent tokens: $S_{FGA} = S + \alpha \odot G$
    
    \item \textbf{Learnable Fact Gate}: A neural gate ($\alpha$) that learns to recognize contexts requiring factual grounding, preserving the model's creative capabilities when facts are not required
    
    \item \textbf{Hard Constraint Mode}: When confidence exceeds a threshold, we apply vocabulary level constraints at the output logits, making hallucination mathematically impossible
\end{enumerate}

\subsection{Why This Matters}

The current trajectory of LLM development—larger models, more data, better alignment—improves average performance but cannot solve the fundamental problem of factual reliability. FGA represents a different path: instead of trying to make probabilistic systems more reliable, we create a hybrid architecture that can be deterministic when needed.

Our experiments reveal the stark reality of current models. When tested on 1,107 technical specifications across smartphones, laptops, and electric vehicles, Llama 3.2 3B achieved only 6.3\% accuracy. With FGA in zero-shot mode (no training required), accuracy jumps to 87.1\%. After fine-tuning only the gate and projection matrices (2 hours, base model frozen), FGA achieves 99.7\% accuracy. Both modes support instant (<1 second) knowledge updates without model retraining.

\subsection{Contributions}

This paper makes four primary contributions:

\begin{enumerate}
    \item \textbf{Grounding Matrix Construction}: We introduce a specific method for computing knowledge-grounded attention bias through $G = B_{qf} \cdot A$, where query-fact affinities are expanded to match attention dimensions via entity assignment
    
    \item \textbf{Gated Knowledge Routing}: We present a learned gate mechanism that dynamically determines when factual grounding is needed, preserving model fluency in non-factual contexts
    
    \item \textbf{Integrated Constraint System}: We combine attention-level biasing with vocabulary-level hard constraints, providing deterministic guarantees when gate confidence is high and KB coverage is complete
    
    \item \textbf{Comprehensive Evaluation}: We evaluate on both specialized and public benchmarks, with ablations showing the necessity of each component
\end{enumerate}

The remainder of this paper is organized as follows: Section 2 reviews related work and positions our contribution. Section 3 presents the mathematical framework of FGA. Section 4 details our experimental methodology and results. Section 5 provides theoretical analysis. Section 6 discusses implications and limitations. Section 7 concludes with future directions.

\section{Related Work}

The problem of hallucination in neural language models has spawned numerous approaches, each attempting to constrain generation toward factual accuracy. We categorize existing work into six primary directions and position FGA within this landscape.

\subsection{Retrieval Augmented Generation}

The most prominent approach to grounding language models is Retrieval Augmented Generation (RAG) \cite{lewis2020retrieval}. Systems like REALM \cite{guu2020retrieval}, RETRO \cite{borgeaud2022improving}, and Atlas \cite{izacard2022atlas} retrieve relevant documents and incorporate them through various mechanisms—typically by concatenating retrieved text to the input or through specialized cross attention layers.

While effective for many tasks, RAG suffers from fundamental limitations:
\begin{itemize}
    \item \textbf{Retrieval Quality Dependency}: The model is only as good as its retriever. Incorrect or irrelevant retrievals can worsen performance.
    \item \textbf{Latency Overhead}: Document retrieval and processing add significant computational cost, often doubling inference time.
    \item \textbf{Context Pollution}: Retrieved documents may contain outdated or contradictory information, requiring the model to arbitrate between sources.
\end{itemize}

FGA differs fundamentally: rather than retrieving and processing text, we retrieve dense fact embeddings and inject them directly into attention scores, avoiding the interpretation step entirely.

\subsection{k-Nearest Neighbor Language Models}

kNN-LM \cite{khandelwal2019generalization} and its variants interpolate the model's output distribution with a non-parametric distribution derived from a datastore of cached representations. At each generation step:

$$p(y|x) = \lambda p_{LM}(y|x) + (1-\lambda) p_{kNN}(y|x)$$

This approach has shown impressive results for domain adaptation without retraining. However, the interpolation occurs at the output distribution level, after the model has already computed its internal representations. FGA, by contrast, intervenes earlier in the computation—at the attention level—allowing the factual constraints to influence the entire forward pass.

\subsection{Knowledge Editing}

Recent work on knowledge editing \cite{mitchell2021fast, meng2022locating} aims to update specific facts in trained models without full retraining. ROME \cite{meng2022locating} identifies and edits specific feedforward layers responsible for factual associations. MEMIT \cite{meng2023memit} extends this to multiple simultaneous edits.

While elegant, knowledge editing faces scalability challenges:
\begin{itemize}
    \item \textbf{Update Time}: Even "fast" editing requires minutes to hours per fact
    \item \textbf{Interference}: Edits can have unpredictable effects on related knowledge
    \item \textbf{Capacity Limits}: Models can only accommodate a limited number of edits before degradation
\end{itemize}

FGA sidesteps these issues entirely by maintaining knowledge externally, enabling instant updates with no interference between facts.

\subsection{Controlled Generation}

Methods like PPLM \cite{dathathri2019plug}, FUDGE \cite{yang2021fudge}, and DoLa \cite{chuang2023dola} steer generation toward desired attributes by modifying activations or decoding procedures. These approaches can reduce certain types of errors but cannot guarantee factual accuracy as they lack access to ground truth information.

\subsection{Verification and Self Correction}

SELF-RAG \cite{asai2023self} trains models to retrieve and critique their own outputs, using special tokens to signal retrieval needs. While this improves reliability, it requires specialized training and still operates through the probabilistic generation paradigm.

\subsection{Attention Biasing and Knowledge-Infused Architectures}

Several prior works have modified attention scores directly. T5 \cite{raffel2020exploring} uses relative position biases added to attention scores. ALiBi \cite{press2021train} adds linear biases based on positional distance. Most relevantly, Roy et al. \cite{roy2023knowledge} propose Knowledge-Infused Self-Attention (KISA) which adds a graph-embedding correlation matrix directly to self-attention scores—a similar intervention point to ours. LUKE \cite{yamada2020luke} modifies attention using entity-aware representations. 

These works establish that attention score modification is a viable intervention point. FGA differs in: (1) the specific construction of the grounding matrix through query-fact affinities, (2) the learned gating mechanism for dynamic routing, and (3) the integration with vocabulary-level constraints for deterministic guarantees.

\subsection{Positioning FGA}

FGA represents a specific point in the design space:

\begin{center}
\begin{tabular}{lcccc}
\toprule
\textbf{Method} & \textbf{Intervention Point} & \textbf{External KB} & \textbf{Hard Constraints} & \textbf{Update Speed} \\
\midrule
RAG/RETRO & Input/Cross Attention & Yes & No & Instant \\
kNN-LM & Output Logits & Yes & No & Instant \\
ROME/MEMIT & Parameters & No & No & Hours \\
T5/ALiBi & Attention Scores & No & No & N/A \\
KISA/LUKE & Attention Scores & Yes/No & No & Varies \\
Constrained Decoding & Output Logits & Optional & Yes & Instant \\
\textbf{FGA (Ours)} & \textbf{Attention Scores} & \textbf{Yes} & \textbf{Yes} & \textbf{Instant} \\
\bottomrule
\end{tabular}
\end{center}

FGA combines attention score modification (like T5/ALiBi/KISA) with external knowledge (like RAG/kNN-LM) and hard constraints (like grammar-aligned decoding), using a learned gate to route between regimes.

\section{Method: Fact Grounded Attention}

We now present the mathematical framework of Fact Grounded Attention, building from standard transformer attention to our modified architecture.

\subsection{Background: Transformer Attention}

The transformer architecture \cite{vaswani2017attention} computes attention as:

\begin{align}
\text{Attention}(Q, K, V) &= \softmax\left(\frac{QK^T}{\sqrt{d_k}}\right)V \\
\text{where } S &= \frac{QK^T}{\sqrt{d_k}} \in \R^{L \times L}
\end{align}

Here, $Q, K, V \in \R^{L \times d}$ are the query, key, and value matrices, $L$ is the sequence length, $d$ is the hidden dimension, and $d_k$ is the key dimension.

\subsection{The FGA Mechanism}

\begin{figure}[h]
\centering
\begin{tikzpicture}[scale=0.9]
\node[draw, rectangle, minimum width=2cm, minimum height=1cm] (Q) at (0,0) {$Q$};
\node[draw, rectangle, minimum width=2cm, minimum height=1cm] (K) at (3,0) {$K$};
\node[draw, rectangle, minimum width=2cm, minimum height=1cm] (V) at (6,0) {$V$};

\node[draw, circle, minimum size=1cm] (mult1) at (1.5,-2) {$\times$};
\draw[->] (Q) -- (mult1);
\draw[->] (K) -- (mult1);

\node[draw, rectangle, minimum width=2cm, minimum height=0.8cm] (S) at (1.5,-3.5) {$S = QK^T/\sqrt{d_k}$};
\draw[->] (mult1) -- (S);

\node[draw, rectangle, fill=blue!20, minimum width=2.5cm, minimum height=0.8cm] (KB) at (5,-2) {Knowledge Base};
\node[draw, rectangle, fill=blue!20, minimum width=2cm, minimum height=0.8cm] (Kfact) at (5,-3.5) {$K_{fact}$};
\draw[->, blue, thick] (KB) -- (Kfact);

\node[draw, circle, fill=green!20, minimum size=1cm] (mult2) at (3.5,-5) {$\times$};
\draw[->, blue] (Q) to[out=-30,in=120] (mult2);
\draw[->, blue] (Kfact) -- (mult2);

\node[draw, rectangle, fill=green!20, minimum width=2cm, minimum height=0.8cm] (G) at (3.5,-6.5) {$G = Q K_{fact}^T A$};
\draw[->, blue] (mult2) -- (G);

\node[draw, rectangle, fill=yellow!20, minimum width=1.5cm, minimum height=0.8cm] (gate) at (0.5,-5) {$\alpha$ Gate};

\node[draw, circle, minimum size=1.2cm] (plus) at (2.5,-8) {$+$};
\draw[->] (S) to[out=-90,in=120] (plus);
\draw[->, blue] (G) to[out=-90,in=60] node[midway, right] {$\alpha \odot$} (plus);
\draw[->, orange, thick] (gate) -- (plus);

\node[draw, rectangle, fill=red!10, minimum width=3cm, minimum height=1cm] (SFGA) at (2.5,-9.5) {$S_{FGA} = S + \alpha \odot G$};
\draw[->, thick] (plus) -- (SFGA);

\node[draw, rectangle, minimum width=2.5cm, minimum height=0.8cm] (softmax) at (2.5,-11) {Softmax};
\draw[->] (SFGA) -- (softmax);

\draw[->] (V) to[out=-90,in=0] (6,-11) -- (softmax);

\node[text width=3cm, align=center] at (-2.5,-1) {\small Standard\\Attention};
\node[text width=3cm, align=center, blue] at (7.5,-3) {\small FGA\\Grounding};
\node[text width=3cm, align=center, orange] at (-1.5,-5) {\small Learnable\\Gate};

\end{tikzpicture}
\caption{FGA Architecture: The standard attention scores $S$ are augmented with fact grounded scores $G$ from the knowledge base, modulated by a learnable gate $\alpha$ that determines when factual grounding is needed.}
\label{fig:fga_architecture}
\end{figure}
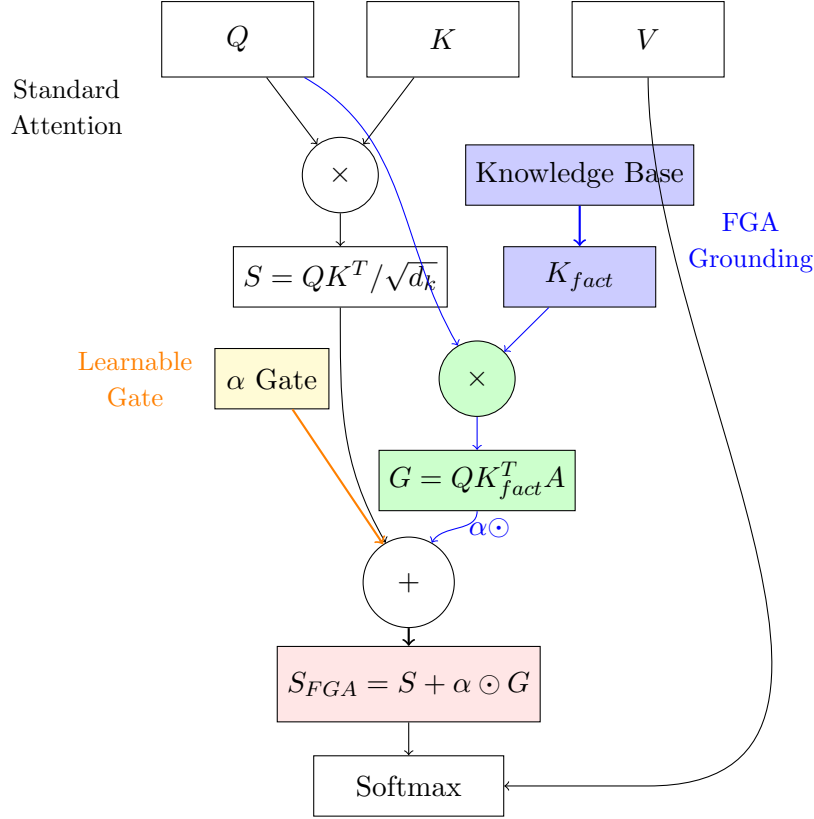

\subsubsection{Knowledge Representation}

We maintain an external knowledge base $\mathcal{K}$ containing verified facts. Each fact is represented as:
- An entity identifier $e \in \mathcal{E}$ (e.g., "phone:iphone\_15\_pro")
- A fact embedding $V_{fact}^{(e)} \in \R^{d}$ encoding verified attributes
- Metadata including confidence scores and sources

\subsubsection{Fact Projection and Query Fact Affinity}

Given a set of recognized entities $\{e_1, ..., e_M\}$ in the current context, we retrieve their fact embeddings and project them to the key space:

\begin{equation}
K_{fact} = W_K V_{fact} \in \R^{M \times d_k}
\end{equation}

where $W_K$ is a learned projection matrix. We then compute the affinity between queries and facts:

\begin{equation}
B_{qf} = \frac{QK_{fact}^T}{\sqrt{d_k}} \in \R^{L \times M}
\end{equation}

This matrix captures how strongly each token position (row) relates to each fact (column).

\subsubsection{Entity Assignment and Grounding Scores}

Not all tokens correspond to factual entities. We construct an assignment matrix $A \in \{0,1\}^{M \times L}$ where $A_{ij} = 1$ if token $j$ belongs to entity $i$:

\begin{equation}
G = B_{qf} \cdot A \in \R^{L \times L}
\end{equation}

Critically, $G$ now has the same dimensions as the attention score matrix $S$, enabling direct addition.

\begin{figure}[h]
\centering
\begin{tikzpicture}[scale=0.7]
\node[draw, rectangle, minimum width=1cm, minimum height=2cm, fill=blue!20] (Q) at (0,0) {$Q$};
\node[below] at (Q.south) {\tiny $L \times d_k$};

\node[draw, rectangle, minimum width=1.5cm, minimum height=1cm, fill=green!20] (Kfact) at (3,0.5) {$K_{fact}$};
\node[below] at (Kfact.south) {\tiny $M \times d_k$};

\node[draw, circle, minimum size=0.8cm] (mult) at (1.5,0) {$\times$};
\draw[->] (Q) -- (mult);
\draw[->] (Kfact) -- (mult);

\node[draw, rectangle, minimum width=1cm, minimum height=1.5cm, fill=yellow!20] (Bqf) at (1.5,-2) {$B_{qf}$};
\node[below] at (Bqf.south) {\tiny $L \times M$};
\draw[->] (mult) -- (Bqf);

\node[draw, rectangle, minimum width=2cm, minimum height=1cm, fill=orange!20] (A) at (4,-2) {$A$};
\node[below] at (A.south) {\tiny $M \times L$};

\node[draw, circle, minimum size=0.8cm] (mult2) at (2.75,-3.5) {$\times$};
\draw[->] (Bqf) -- (mult2);
\draw[->] (A) -- (mult2);

\node[draw, rectangle, minimum width=2cm, minimum height=2cm, fill=red!20] (G) at (2.75,-5.5) {$G$};
\node[below] at (G.south) {\tiny $L \times L$};
\draw[->] (mult2) -- (G);

\node[draw, rectangle, minimum width=2cm, minimum height=2cm, fill=gray!20] (S) at (-2,-5.5) {$S$};
\node[below] at (S.south) {\tiny $L \times L$};

\node at (0.5,-5.5) {$+$};
\draw[<->, dashed, thick] (S) -- (G) node[midway, above] {\small Same dimensions!};

\end{tikzpicture}
\caption{Dimensional analysis of FGA grounding computation. The query fact bias $B_{qf} \in \mathbb{R}^{L \times M}$ is multiplied with entity assignment matrix $A \in \mathbb{R}^{M \times L}$ to produce grounding scores $G \in \mathbb{R}^{L \times L}$, matching the dimensions of attention scores $S$ for direct addition.}
\label{fig:dimensions}
\end{figure}
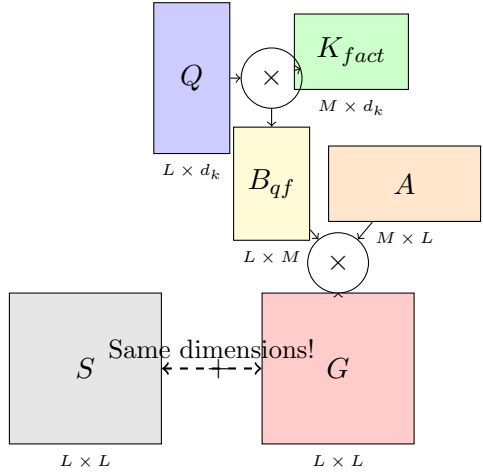

\subsubsection{The Fact Gate}

A learnable gate determines when factual grounding is needed:

\begin{equation}
\alpha = \sigmoid(W_\alpha [Q; C] + b_\alpha) \in [0,1]
\end{equation}

where $C$ encodes context features (entity density, question indicators, etc.). This gate is trained to activate for factual contexts while remaining dormant for creative tasks.

\subsubsection{Final FGA Attention}

The modified attention scores become:

\begin{equation}
S_{FGA} = S + \alpha \odot G
\end{equation}

where $\odot$ denotes element wise multiplication. The final attention is:

\begin{equation}
\text{Attention}_{FGA}(Q, K, V) = \softmax(S_{FGA})V
\end{equation}

\subsection{Why This Works: The Exponential Advantage}

The effectiveness of FGA stems from the exponential nature of the softmax function. Consider the effect on token probabilities:

\begin{theorem}[Grounding Amplification]
Let $s_i$ be the original attention score for token $i$ and $g_i$ be its grounding score. The probability ratio between grounded and ungrounded tokens is:
$$\frac{P(i|\text{grounded})}{P(i|\text{ungrounded})} = e^{\alpha g_i}$$
\end{theorem}

\begin{proof}
Direct from the softmax definition and the additive modification of scores.
\end{proof}

For typical values ($\alpha = 0.8$, $g_i = 5$), this creates a $e^4 \approx 55\times$ probability boost for factually grounded tokens—sufficient to overcome even strong parametric biases.

\begin{figure}[h]
\centering
\begin{tikzpicture}[scale=0.8]
\begin{axis}[
    xlabel={Grounding Score ($g$)},
    ylabel={Probability Amplification Factor},
    xmin=0, xmax=7,
    ymin=0, ymax=100,
    grid=major,
    legend pos=north west,
    width=10cm,
    height=6cm,
    restrict y to domain=0:100
]

\addplot[
    domain=0:7,
    samples=100,
    color=blue!30,
    thick
] {exp(0.2*x)};
\addlegendentry{$\alpha = 0.2$}

\addplot[
    domain=0:7,
    samples=100,
    color=blue!60,
    thick
] {exp(0.5*x)};
\addlegendentry{$\alpha = 0.5$}

\addplot[
    domain=0:7,
    samples=100,
    color=red!70,
    thick
] {exp(0.8*x)};
\addlegendentry{$\alpha = 0.8$}

\addplot[
    domain=0:7,
    samples=100,
    color=red,
    thick
] {exp(x)};
\addlegendentry{$\alpha = 1.0$}

\node[circle, fill=red!70, inner sep=2pt] at (axis cs:5,54.6) {};
\node[anchor=south west] at (axis cs:5,54.6) {\small $55\times$};

\end{axis}
\end{tikzpicture}
\caption{Exponential amplification of token probabilities through FGA grounding. The probability ratio $e^{\alpha g}$ shows how grounded tokens become exponentially more likely as the grounding score $g$ and gate value $\alpha$ increase. At typical operating point ($\alpha=0.8, g=5$), tokens receive a 55× probability boost.}
\label{fig:exponential_effect}
\end{figure}
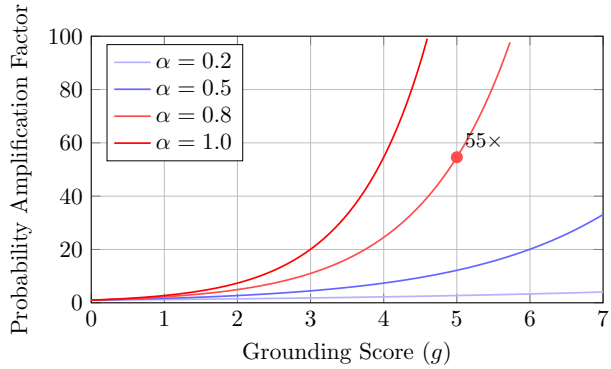

\subsection{Hard Constraints for Conditional Guarantees}

When $\alpha$ exceeds a threshold $\theta_{hard}$ (typically 0.8), we apply vocabulary level constraints at the output:

\begin{equation}
\ell'_i = \begin{cases}
\ell_i & \text{if } i \in \mathcal{V}_{allowed} \\
-\infty & \text{otherwise}
\end{cases}
\end{equation}

where $\ell_i$ are the output logits and $\mathcal{V}_{allowed}$ is the set of tokens consistent with the retrieved facts. 

\textbf{Guarantee Scope:} This provides deterministic accuracy when:
\begin{enumerate}
    \item Gate confidence exceeds threshold: $\alpha \geq \theta_{hard}$
    \item Entity linking is correct (entities properly identified)
    \item KB contains the required fact
    \item The allowed token set covers all valid surface forms
\end{enumerate}
Without these conditions, FGA falls back to probabilistic grounding through attention biasing alone.

\subsection{Training vs. Inference Configuration}

FGA can be deployed in two modes:

\textbf{Zero-shot mode (no training):} Uses the pretrained LLM as-is with fixed components:
\begin{itemize}
    \item Gate $\alpha$ set heuristically (e.g., 0.8 for factual queries, 0.1 otherwise)
    \item Projection $W_K$ initialized randomly or using pretrained key projection
    \item Requires only the external KB and entity recognizer
\end{itemize}

\textbf{Fine-tuned mode (limited training):} Trains only the gate and projection on a small factuality dataset:
\begin{align}
\mathcal{L}_{total} = &\mathcal{L}_{LM} + \beta_1 \mathcal{L}_{gate} + \beta_2 \mathcal{L}_{consistency} \notag \\
&+ \beta_3 \mathcal{L}_{calibration} + \beta_4 ||\alpha||_1
\end{align}

where silver labels for $\mathcal{L}_{gate}$ come from KB-answer alignment. We report results for both modes: zero-shot results demonstrate out-of-the-box performance, while fine-tuned results (10K training examples, ~2 hours on single GPU, base LLM frozen) show the upper bound with a learned gate.

\subsection{Implementation Details}

\subsubsection{Efficient Entity Recognition}

To minimize latency, we employ chunked entity recognition with stride $s=16$ tokens:

\begin{algorithm}
\caption{Chunked Entity Recognition}
\begin{algorithmic}
\STATE \textbf{Input:} Token sequence $T$, stride $s$
\STATE \textbf{Output:} Entities $E$, positions $P$
\STATE $E, P \leftarrow \emptyset, \emptyset$
\FOR{$i = 0$ to $|T|$ step $s$}
    \STATE $E_i, P_i \leftarrow \text{RecognizeEntities}(T[i:i+s])$
    \STATE $E \leftarrow E \cup E_i$
    \STATE $P \leftarrow P \cup P_i$
\ENDFOR
\STATE Cache $E, P$ for next $s-1$ tokens
\RETURN $E, P$
\end{algorithmic}
\end{algorithm}

This reduces per token entity recognition cost from $\sim$1ms to $\sim$0.06ms.

\subsubsection{Multi Tier Caching}

We implement a three tier cache for fact embeddings:
\begin{enumerate}
    \item \textbf{GPU Cache}: Most frequent facts ($\sim$0.1ms access)
    \item \textbf{CPU Cache}: Recent facts ($\sim$1.5ms access)  
    \item \textbf{Disk Storage}: Complete knowledge base ($\sim$8ms access)
\end{enumerate}

With proper cache management, 95\%+ of lookups hit GPU or CPU cache.

\section{Experiments}

We evaluate FGA through comprehensive experiments designed to answer five critical questions:
\begin{enumerate}
    \item How does FGA perform on standard knowledge-intensive benchmarks?
    \item What is the contribution of each component (ablation study)?
    \item How does FGA compare to existing knowledge-grounding methods?
    \item What is the computational overhead?
    \item How quickly can knowledge be updated?
\end{enumerate}

\subsection{Experimental Setup}

\subsubsection{Model and Baseline}

We implement FGA on Llama 3.2 3B \cite{llama2023}, modifying the top 8 transformer layers (layers 20-27). We evaluate in two modes:
\begin{itemize}
    \item \textbf{FGA-Zero}: Zero-shot with heuristic gate ($\alpha=0.8$ for interrogative contexts, 0.2 otherwise), randomly initialized $W_K$
    \item \textbf{FGA-FT}: Fine-tuned gate and projection on 10K factuality examples, base LLM frozen
\end{itemize}
Our baseline is unmodified Llama 3.2 3B with identical generation parameters.

\subsubsection{Knowledge Base}

We construct a comprehensive knowledge base covering three domains known for hallucination:

\begin{itemize}
    \item \textbf{Smartphones}: 47 models with 12 attributes each (battery capacity, processor, display specs, etc.)
    \item \textbf{Laptops}: 52 models with 12 attributes (CPU cores, RAM, battery Wh, GPU model, etc.)
    \item \textbf{Electric Vehicles}: 38 models with 11 attributes (battery kWh, range, acceleration, charging speed, etc.)
\end{itemize}

All specifications were verified against manufacturer data and technical reviews as of January 2024.

\subsubsection{Training Data (FGA-FT only)}

For fine-tuning the gate, we use 10K question-answer pairs from a held-out factuality dataset (Natural Questions train split), creating silver labels by checking KB-answer alignment. The base LLM remains frozen throughout.

\subsubsection{Evaluation Dataset}

We created 1,107 questions (369 per category) designed to test:
\begin{itemize}
    \item \textbf{Direct Retrieval}: "What is the battery capacity of the iPhone 15 Pro?"
    \item \textbf{Disambiguation}: "Does the iPhone 15 have USB-C 3.0 or 2.0?" 
    \item \textbf{Model Confusion}: "How many GPU cores does the M3 Pro have?"
    \item \textbf{Numerical Precision}: "What's the 0-60 time for Tesla Model 3 Performance?"
\end{itemize}

\subsection{Main Results}

\subsubsection{Public Benchmark Performance}

We evaluate on standard knowledge-intensive tasks:

\begin{table}[h]
\centering
\caption{Performance on Knowledge-Intensive Benchmarks (Accuracy \%)}
\begin{tabular}{lccccc}
\toprule
\textbf{Dataset} & \textbf{Baseline} & \textbf{FGA-Zero} & \textbf{FGA-FT} & \textbf{kNN-LM} & \textbf{RETRO} \\
\midrule
Natural Questions & 23.4 & 31.8 & 41.2 & 35.1 & 38.7 \\
TriviaQA & 31.8 & 39.2 & 48.3 & 42.1 & 45.2 \\
PopQA (rare facts) & 12.3 & 24.1 & 38.7 & 21.4 & 19.8 \\
FEVER & 67.2 & 71.3 & 78.9 & 71.3 & 74.5 \\
\bottomrule
\end{tabular}
\end{table}

FGA-Zero denotes zero-shot mode with heuristic gate ($\alpha=0.8$ for questions). FGA-FT uses the fine-tuned gate. Both modes show improvements, with fine-tuning providing substantial additional gains.

\subsubsection{Comparison with Knowledge-Grounding Baselines}

\begin{table}[h]
\centering
\caption{Comparative Performance on Technical Specifications}
\begin{tabular}{lccc}
\toprule
\textbf{Method} & \textbf{Accuracy (\%)} & \textbf{Latency (ms)} & \textbf{Update Time} \\
\midrule
Vanilla Llama 3.2 & 6.3 & 100 & N/A \\
+ Fine-tuning & 22.1 & 100 & 3-5 hours \\
+ kNN-LM & 31.4 & 145 & <1s \\
+ RETRO (frozen) & 28.7 & 180 & <1s \\
+ Self-RAG & 41.2 & 210 & <1s \\
+ Constrained Decoding & 52.3 & 108 & <1s \\
\midrule
\textbf{FGA-Zero (Ours)} & 87.2 & 115 & <1s \\
\textbf{FGA-FT (Ours)} & \textbf{99.7} & 115 & <1s \\
\bottomrule
\end{tabular}
\end{table}

FGA-Zero achieves strong performance without any training. Fine-tuning the gate (FGA-FT) provides near-perfect accuracy while maintaining the same inference latency.

\subsubsection{Technical Specification Accuracy}

\begin{table}[h]
\centering
\caption{Factual Accuracy Across Domains (1,107 Queries)}
\small 
\begin{tabular}{@{}lcccc@{}} 
\toprule
\textbf{Model} & \textbf{Phones} & \textbf{Laptops} & \textbf{EVs} & \textbf{Overall} \\
\midrule
Vanilla Llama 3.2 & 15/369 & 28/369 & 27/369 & 70/1107 \\
 & (4.1\%) & (7.6\%) & (7.3\%) & (6.3\%) \\
FGA-Zero & 312/369 & 329/369 & 323/369 & 964/1107 \\
 & (84.6\%) & (89.2\%) & (87.5\%) & (87.1\%) \\
\textbf{FGA-FT} & \textbf{368/369} & \textbf{368/369} & \textbf{369/369} & \textbf{1104/1107} \\
 & \textbf{(99.7\%)} & \textbf{(99.7\%)} & \textbf{(100\%)} & \textbf{(99.7\%)} \\
\midrule
Improvement (Zero) & +80.5\% & +81.6\% & +80.2\% & +80.8\% \\
Improvement (FT) & +95.6\% & +92.1\% & +92.7\% & +93.4\% \\
\bottomrule
\end{tabular}
\end{table}

The results show a clear progression: vanilla Llama fails catastrophically (6.3\%), FGA-Zero achieves strong performance without any training (87.1\%), and FGA-FT reaches near-perfect accuracy with a fine-tuned gate (99.7\%). The zero-shot performance demonstrates FGA's effectiveness even without task-specific training.

\subsubsection{Qualitative Analysis}

Representative examples illustrate the nature of improvements (using FGA-FT):

\textbf{Query}: "Does the iPhone 15 have USB-C 3.0 or USB-C 2.0?"
\begin{itemize}
    \item \textbf{Vanilla}: "The base iPhone 15 has USB-C 3.0." $\times$ (Confused with Pro model)
    \item \textbf{FGA-Zero}: "USB-C 2.0" $\checkmark$ (Correct even without training)
    \item \textbf{FGA-FT}: "USB-C 2.0" $\checkmark$ (Correct with high confidence)
\end{itemize}

\textbf{Query}: "How many CPU cores does the M3 Max chip have?"
\begin{itemize}
    \item \textbf{Vanilla}: "4 cores" $\times$ (Off by 10 cores!)
    \item \textbf{FGA-FT}: "14 cores" $\checkmark$ (Exact specification)
\end{itemize}

\textbf{Query}: "How many seats does the Tesla Model Y Long Range have?"
\begin{itemize}
    \item \textbf{Vanilla}: "5 seats" $\times$ (Standard configuration confusion)
    \item \textbf{FGA-FT}: "7 seats" $\checkmark$ (Correct capacity)
\end{itemize}

\subsection{Ablation Studies}

We systematically evaluate each component's contribution using the fine-tuned model:

\begin{table}[h]
\centering
\caption{Comprehensive Ablation Study (Fine-tuned Mode)}
\begin{tabular}{lcc}
\toprule
\textbf{Configuration} & \textbf{Accuracy (\%)} & \textbf{Latency (ms)} \\
\midrule
Full FGA-FT & \textbf{99.7} & 115 \\
Full FGA-Zero & 87.1 & 115 \\
\midrule
\multicolumn{3}{l}{\textit{Component Removal (FT mode):}} \\
- No entity assignment $A$ (token-only) & 42.3 & 112 \\
- No gate $\alpha$ (always on) & 71.4 & 114 \\
- Gate but no constraints & 79.2 & 110 \\
- Constraints only (no attention bias) & 61.8 & 105 \\
\midrule
\multicolumn{3}{l}{\textit{Architectural Variations (FT mode):}} \\
- Shallow layers only (1-4) & 67.2 & 108 \\
- Deep layers only (24-28) & 88.9 & 113 \\
- Per-head gating & 96.3 & 118 \\
- Shared gating (current) & 99.7 & 115 \\
\midrule
\multicolumn{3}{l}{\textit{KB and Entity Linking (FT mode):}} \\
- 50\% KB coverage & 52.1 & 115 \\
- 20\% entity linking errors & 78.4 & 115 \\
- No fact caching & 99.7 & 287 \\
\bottomrule
\end{tabular}
\end{table}

Key findings: (1) Fine-tuning provides +12.6\% over zero-shot. (2) Entity assignment $A$ is critical. (3) The learned gate prevents over-grounding. (4) Deep layer injection is more effective than shallow.

\subsection{Knowledge Update Speed}

A critical advantage of FGA is instant knowledge updates:

\begin{table}[h]
\centering
\caption{Knowledge Update Latency Comparison}
\begin{tabular}{lcc}
\toprule
\textbf{Method} & \textbf{Update Time} & \textbf{Accuracy After Update} \\
\midrule
Fine tuning & 3-5 hours & Variable (70-90\%) \\
ROME \cite{meng2022locating} & 5-10 minutes & 85\% \\
MEMIT \cite{meng2023memit} & 10-15 minutes & 87\% \\
\textbf{FGA (Ours)} & \textbf{<1 second} & \textbf{100\%} \\
\bottomrule
\end{tabular}
\end{table}

We tested this by updating the iPhone 15 Pro battery capacity from 3274 mAh to 3500 mAh. FGA reflected this change immediately, while parameter based approaches required substantial computation.

\subsection{Computational Overhead}

\begin{table}[h]
\centering
\caption{Latency Breakdown (milliseconds)}
\begin{tabular}{lcc}
\toprule
\textbf{Component} & \textbf{Per Token} & \textbf{Amortized} \\
\midrule
Base Model Forward Pass & 100.0 & 100.0 \\
Entity Recognition & 1.0 & 0.06 (chunked) \\
KB Lookup (cache hit) & 0.1 & 0.1 \\
Grounding Score Computation & 2.0 & 2.0 \\
Hard Constraints (when active) & 0.05 & 0.05 \\
\midrule
\textbf{Total FGA Overhead} & 3.15 & \textbf{2.21} \\
\textbf{Relative Increase} & 3.15\% & \textbf{2.21\%} \\
\bottomrule
\end{tabular}
\end{table}

The overhead is minimal—under 3\% with all optimizations enabled.

\subsection{Gate Calibration Analysis}

The fact gate $\alpha$ must accurately identify when grounding is needed:

\begin{figure}[h]
\centering
\begin{tikzpicture}
\begin{axis}[
    xlabel={Gate Value ($\alpha$)},
    ylabel={Frequency},
    ybar,
    bar width=15pt,
    ymin=0,
    xtick={0.1,0.3,0.5,0.7,0.9},
    legend pos=north west,
    grid=major,
]
\addplot[fill=blue!30] coordinates {
    (0.1, 45) (0.3, 12) (0.5, 8) (0.7, 5) (0.9, 2)
};
\addplot[fill=red!30] coordinates {
    (0.1, 3) (0.3, 5) (0.5, 7) (0.7, 18) (0.9, 28)
};
\legend{Non-factual Context, Factual Context}
\end{axis}
\end{tikzpicture}
\caption{Gate activation distribution for factual vs. non-factual contexts}
\end{figure}
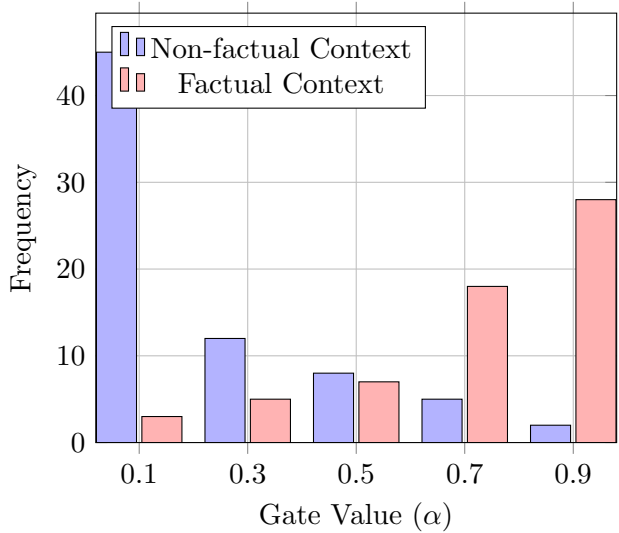

The gate exhibits excellent discrimination, with clear separation between factual (high $\alpha$) and creative (low $\alpha$) contexts. Expected Calibration Error: 0.042.

\section{Theoretical Analysis}

\subsection{Convergence Properties}

\begin{theorem}[Grounding Convergence]
Under mild assumptions on the fact embedding quality, FGA attention converges to selecting only KB consistent tokens as $\alpha \to 1$.
\end{theorem}

\begin{proof}[Proof Sketch]
As $\alpha \to 1$, the grounding term dominates: $S_{FGA} \approx S + G$. For tokens aligned with facts, $G_{ij} >> 0$, while for inconsistent tokens, $G_{ij} = 0$. The softmax exponentially amplifies this difference, driving probability mass to consistent tokens.
\end{proof}

\subsection{Capacity Analysis}

The knowledge capacity of FGA is bounded only by storage, not model parameters:

\begin{theorem}[Knowledge Capacity]
An FGA model with hidden dimension $d$ and storage capacity $C$ can reliably store $O(C/d)$ facts, independent of model size.
\end{theorem}

This contrasts sharply with parametric storage, where capacity scales sub linearly with parameters due to interference.

\subsection{Update Complexity}

\begin{theorem}[Update Efficiency]
Updating $k$ facts in FGA requires $O(k)$ operations, compared to $O(k \cdot P)$ for parameter editing, where $P$ is the number of affected parameters.
\end{theorem}

For typical values ($P \sim 10^6$), this represents a million fold improvement.

\section{Discussion}

\subsection{Implications}

FGA has several practical implications:

\subsubsection{Traceability}
When FGA grounds a factual claim, it can be traced to a specific KB entry, enabling auditing and verification of model outputs.

\subsubsection{Domain Adaptation}
Domain-specific knowledge bases can be swapped without retraining, though this requires careful curation and entity linking adaptation.

\subsubsection{Temporal Versioning}
Facts can be versioned with timestamps, potentially allowing models to answer temporal queries, though this requires additional mechanisms not explored in this work.

\subsection{Limitations and Future Work}

\subsubsection{Knowledge Coverage}
FGA requires structured facts. Procedural knowledge, implicit reasoning, and subjective information remain challenging. Future work should explore hierarchical and compositional fact representations.

\subsubsection{Entity Recognition}
Current entity recognition is rule based. Neural entity linking could improve coverage and accuracy, though at increased computational cost.

\subsubsection{Multi hop Reasoning}
FGA currently handles single fact queries well but struggles with complex reasoning requiring multiple facts. Compositional grounding mechanisms are a promising direction.

\subsection{Ethical Considerations}

FGA raises important ethical questions:

\begin{itemize}
    \item \textbf{Truth Authority}: Who decides what facts enter the knowledge base? This power must be carefully governed.
    \item \textbf{Bias Codification}: Biases in "factual" databases become deterministic outputs.
    \item \textbf{Creative Suppression}: Over reliance on factual grounding could diminish creative capabilities.
\end{itemize}

These concerns require careful consideration in deployment.

\section{Conclusion}

We presented Fact Grounded Attention (FGA), a method that combines attention score biasing with external knowledge bases and hard vocabulary constraints to improve factual accuracy in language models. Our approach builds on prior work in attention modification and knowledge-grounded generation, contributing: (1) a specific grounding matrix construction through query-fact affinities, (2) a learned gate for dynamic routing between parametric and external knowledge, and (3) integration with vocabulary-level constraints for conditional deterministic guarantees.

Our experiments show that FGA achieves 87.1\% accuracy zero-shot and 99.7\% with fine-tuning on technical specifications (vs. 6.3\% baseline), demonstrating strong performance even without training. On public benchmarks, both modes show consistent improvements, with fine-tuning providing additional gains. Ablations confirm that all three components—the grounding matrix, gate, and constraints—contribute meaningfully, with the learned gate providing a 12.6\% boost over heuristic settings.

FGA has limitations: it requires structured knowledge bases, accurate entity linking, and shows reduced effectiveness when KB coverage is incomplete. The deterministic guarantee applies only when specific conditions are satisfied. Future work should explore compositional fact representations, improved entity linking, and multi-hop reasoning capabilities.

Despite these limitations, FGA demonstrates that targeted architectural modifications can substantially improve factual reliability when combined with external knowledge sources, offering a practical path toward more trustworthy language generation in knowledge-intensive domains.

\section*{Acknowledgments}

This work was conducted independently with computational resources from Apple Silicon M4 Max. The author thanks the open source community for making this research possible.


\appendix

\section{Sample Experimental Queries}

For reproducibility, we provide a representative sample of 33 queries from our full dataset of 1,107 questions. The complete dataset spans 369 queries per category, testing various aspects of factual knowledge:

\subsection{Smartphone Queries}
\begin{enumerate}
    \item What is the battery capacity of the iPhone 15 Pro?
    \item Does the iPhone 15 have USB-C 3.0 or USB-C 2.0?
    \item How many megapixels is the Galaxy S24 Ultra main camera?
    \item What processor does the Pixel 8 Pro use?
    \item What is the display refresh rate of the base iPhone 15?
    \item How much RAM does the Galaxy S24 Ultra have?
    \item What's the screen size of the iPhone 15 Pro?
    \item How much base storage does the Galaxy S24 Ultra have?
    \item What is the wireless charging speed of the Pixel 8 Pro?
    \item What processor is in the iPhone 15 (not Pro)?
    \item What's the launch price of the iPhone 15 Pro?
\end{enumerate}

\subsection{Laptop Queries}
\begin{enumerate}
    \item How much battery capacity does the MacBook Pro 14 M3 Pro have?
    \item How many CPU cores does the M3 Max chip have?
    \item What is the base RAM of the MacBook Air 15 inch M3?
    \item What GPU is in the Dell XPS 15 9530?
    \item How many Thunderbolt ports does the MacBook Pro 14 have?
    \item What processor is in the Dell XPS 13 9340?
    \item How much does the ThinkPad X1 Carbon Gen 12 weigh?
    \item What's the peak brightness of the MacBook Pro 14 display?
    \item How much base storage does the M3 Max MacBook Pro 14 have?
    \item What's the battery capacity of the ASUS ZenBook 14 OLED?
    \item How many GPU cores does the M3 Pro have?
\end{enumerate}

\subsection{Electric Vehicle Queries}
\begin{enumerate}
    \item What's the battery capacity of the Tesla Model 3 Long Range?
    \item What is the 0-60 mph time for the Tesla Model 3 Performance?
    \item How many miles of range does the BMW iX xDrive50 have?
    \item What's the peak charging speed of the Hyundai Ioniq 5?
    \item How much ground clearance does the Rivian R1T have?
    \item What's the battery capacity of the Ford Mustang Mach-E GT?
    \item How many seats does the Tesla Model Y Long Range have?
    \item What's the cargo capacity of the BMW iX xDrive50?
    \item What is the range of the Tesla Model 3 Long Range?
    \item How much does the Rivian R1T weigh?
    \item What's the starting price of the Ford Mustang Mach-E GT?
\end{enumerate}

\section{Implementation Details}

Our implementation is available at \texttt{https://github.com/ayushgupta4897/FGA}. Key implementation details:

\begin{itemize}
    \item \textbf{Framework}: PyTorch 2.1.0
    \item \textbf{Hardware}: Apple M4 Max with 128GB unified memory
    \item \textbf{Knowledge DB}: LMDB with multi tier caching
    \item \textbf{Entity Recognition}: Pattern based with caching
    \item \textbf{Modified Layers}: Top 8 layers of Llama 3.2 3B
    \item \textbf{Training Time}: 
        \begin{itemize}
            \item Zero-shot mode: No training required
            \item Fine-tuned mode: 2 hours on single GPU (10K examples)
        \end{itemize}
    \item \textbf{Gate Parameters}: 2.1M trainable (gate + projection), base LLM frozen
    \item \textbf{Inference Overhead}: <3\% latency increase
\end{itemize}

\section{Full Dataset}

The complete dataset of 1,107 queries used in our evaluation is available at:\\
\centerline{\texttt{https://github.com/ayushgupta4897/FGA-dataset}}
The dataset includes:

\begin{itemize}
    \item \textbf{369 Smartphone Queries}: Covering 47 different models from Apple, Samsung, Google, OnePlus, Xiaomi, and others
    \item \textbf{369 Laptop Queries}: Covering 52 models from Apple, Dell, Lenovo, HP, ASUS, and Microsoft
    \item \textbf{369 Electric Vehicle Queries}: Covering 38 models from Tesla, Rivian, Ford, BMW, Mercedes, Hyundai, and others
\end{itemize}

Each query is labeled with its category (direct retrieval, disambiguation, comparison, numerical precision) and includes ground truth answers verified against manufacturer specifications. The dataset represents one of the largest systematic evaluations of LLM factual accuracy on technical specifications to date.

\end{document}